\documentclass[review,3p,11pt]{elsarticle}

\usepackage{amssymb}
\usepackage{amsthm} 

\newtheorem{thm}{Theorem}
\newtheorem{lem}{Lemma}
\newtheorem{rem}{Remark}

\newtheorem{problem}{Problem}

\usepackage{enumitem}

\usepackage{graphicx}
\usepackage{amsmath}

\usepackage[breaklinks]{hyperref}
\usepackage{xurl}

\usepackage{booktabs}

\usepackage{float}
\usepackage{sidecap}
\usepackage{bbm}
\usepackage{mathtools}
\usepackage[font=scriptsize]{caption}
\usepackage{framed}

\DeclareMathOperator*{\argmin}{argmin}
\DeclareMathOperator*{\BV}{BV}

\newcommand{\R}{\mathbb{R}}
\newcommand{\N}{\mathbb{N}}
\newcommand{\M}{\mathcal{M}}

\newcommand{\la}{\lambda}
\newcommand{\eps}{\epsilon}
\newcommand{\set}[1]{\{#1\}}
\newcommand{\abs}[1]{\lvert#1\rvert}
\newcommand{\norm}[1]{\lVert#1\rVert}
\renewcommand{\div}{\operatorname{div}}

\newcommand{\X}{\mathbb{X}}
\newcommand{\Y}{\mathbb{Y}}

\newcommand{\prox}{\operatorname{prox}}
\newcommand{\dom}{\operatorname{dom}}
\newcommand{\peri}{\operatorname{per}}
\newcommand{\diag}{\operatorname{diag}}

\DeclarePairedDelimiter{\inner}{\langle}{\rangle}

\newcommand{\U}{\mathbb A}
\newcommand{\Hr}{\mathbb H}
\newcommand{\Om}{\Omega}

\newcommand{\dist}{\mathcal{D}}

\newcommand{\func}{\mathcal{E}}
\newcommand{\funcR}{\mathcal{R}}
\newcommand{\reg}{\mathcal{Q}}
\newcommand{\F}{\mathbb{F}}

\newcommand{\pdK}{K}
\newcommand{\pdF}{\mathcal{F}}
\newcommand{\pdH}{\mathcal{H}}
\newcommand{\pdG}{\mathcal{G}}

\newcommand{\diffD}{\boldsymbol{M}}
\newcommand{\diffC}{\boldsymbol{M}}
\newcommand{\diffB}{m}

\newcommand{\uu}{\boldsymbol{u}}
\newcommand{\vv}{\boldsymbol{v}}
\newcommand{\ww}{\boldsymbol{w}}
\newcommand{\ff}{\boldsymbol{\phi}}
\newcommand{\ca}{\boldsymbol{a}}
\newcommand{\cb}{\boldsymbol{b}}

\newcommand{\xx}{x}
\newcommand{\yy}{y}

\newcommand{\averageC}{\boldsymbol{A}}
\newcommand{\averageA}{A}

\newcommand{\grad}{\nabla}
\newcommand{\gradd}{\boldsymbol{\nabla}}

\usepackage[colorinlistoftodos]{todonotes}

\usepackage[ruled,vlined]{algorithm2e}

\begin{document}

\begin{frontmatter}

\title{Lifting-based variational multiclass segmentation algorithm: design, convergence analysis, and implementation with applications in medical imaging}

\author[1,2]{Nadja Gruber}
\author[3]{Johannes Schwab}
\author[1]{Sébastien Court}
\author[4]{Elke R. Gizewski}

\author[1]{Markus Haltmeier\corref{coraut}}
\ead{markus.haltmeier@uibk.ac.at}
\cortext[coraut]{Corresponding author}


\affiliation[1]{organization={Department of Mathematics, University of Innsbruck},country= {Austria}}

\affiliation[2]{organization={VASCage-Research Centre on Vascular Ageing and Stroke},            addressline= {Innsbruck},
country={Austria}}

\affiliation[3]{organization={MRC Laboratory of Molecular Biology, Cambridge},            country={UK}}

\affiliation[4]{organization={Department of Neuroradiology, Medical University of Innsbruck},
country={Austria}}

\numberwithin{equation}{section}
\numberwithin{thm}{section}
\numberwithin{lem}{section}
\numberwithin{rem}{section}
\numberwithin{cor}{section}
\numberwithin{problem}{section}

\allowdisplaybreaks

\begin{abstract}
We propose, analyze and realize a variational multiclass segmentation scheme that partitions a given image into multiple regions exhibiting specific properties. Our method determines multiple functions that encode the segmentation regions by minimizing an energy functional combining information from different channels. Multichannel image data can be obtained by lifting the image into a higher dimensional feature space using specific multichannel filtering or may already be provided by the imaging modality under consideration, such as an RGB image or multimodal medical data.  Experimental results show that the proposed method performs well in various scenarios. In particular, promising  results are  presented for two medical applications involving  classification of brain abscess  and tumor growth, respectively. As main theoretical contributions, we prove the existence of global minimizers of the proposed energy functional and show its stability and convergence with respect to noisy inputs. In particular, these results also apply to the special case of binary segmentation, and these results are also novel in this particular situation.
\end{abstract}

\medskip

\end{frontmatter}

\section{Introduction}
\label{intro}

The aim of segmentation is to divide an image defined on some bounded domain $\Omega  \subseteq \R^2$ into subregions that are homogeneous with regard to certain characteristics, such as intensity, color, or texture. This process plays a fundamental role in various semantic applications such as object recognition, classification, or medical diagnostics. Many successful image segmentation methods  are based on variational  and active contour models \citep{morel2012variational,mumford1989optimal,scherzer2009variational}, which have in common that they find optimal segmentations by minimizing an objective function, which generally depends on the given image and the features used to identify the different  regions to be segmented.

In the simplest case of binary segmentation, the goal is to divide a given image into two regions, one that represents the object to be recognized and the second one that represents background. A particularly popular approach in that regard is the Chan-Vese model \citep{chan2001active}, that is based on a level-set function $\phi \colon \Omega \to \R$ defining two regions $\{x \in \Om  \mid \phi(x) \leq 0\}$ and $\{x \in \Om   \mid \phi(x) >0\}$.  The  level-set function $\phi$ is constructed  by minimizing a certain energy functional combining regularity of the segmentation region and fitting to the provided input image. The extension to non-binary segmentation (or multiclass segmentation) is challenging due to several reasons. For example, using  $m$ level-set functions naturally yields $2^m$ separate regions (corresponding to all possible combinations of overlaps of the individual level-sets), which may be different from the desired number of  regions to be segmented.  Furthermore, in a naive approach, the segmentation function is applied to the unfiltered original intensity image. In practice, however, other characteristics like texture or color may be better suited to separate individual regions; see \citep{martin2004learning,randen1999filtering,rousson2003active}.   

In this paper, we present a variational framework for multiclass segmentation based on lifting the image to be segmented into a space of $K$-channel images (feature maps), on which we apply a proposed variational segmentation functional. {The individual channels are generated to well separate the $k$-th class from the remaining classes. This can be achieved, for example, by applying multiple filters or by combining information from naturally occurring channels, as in multimodal medical imaging.  Even imperfect separation of the different regions by means of the feature maps can be corrected by the actual segmentation process.}  It should be noted that the generation of vector-valued feature maps prior to actual segmentation is not a new proposal; see e.g. \citep{bae2017convex,kiechle2018model,storath2014unsupervised} and the references there. However, the specific segmentation functional, the convergence analysis, and the proposed minimization algorithm are new, and are the key contributions of the present work.

\subsection{Proposed lifting-based segmentation}

Let $\F$ be a space of functions $f \colon \Omega \to \M$ with values in some manifold $\M$. The manifold  $\M$ is generic and, for example,  can consist of (a subset of)  the real numbers $\R$ in the case of gray-value images, be $\R^d$  in the case  of multimodal imaging, or it can consist of certain tensors such as in diffusion tensor imaging. After feature lifting, the image values will be elements in $\R^K$, with $K+1$ denoting the number of distinct classes. Our aim is  solving the  following  task:

\begin{problem}[Multiclass segmentation] \label{pr:segment}
Based on specific pre-defined characteristics of the image $f \in \F$,  construct a partitioning $\Omega = \bigcup_{k=0}^K  \Sigma_k$ of the domain $\Omega$ into $K+1$ disjoint regions (classes). Each  of the regions $\Sigma_1, \dots , \Sigma_K$ represents a specific structure or objects in the given image and $\Sigma_0$ represents background.
\end{problem}

Let $\BV(\Omega)$ denote the space of all integrable functions $u \colon \Omega  \to \R$ with bounded total variation $\abs{u}_{\rm TV}$, write $\uu = (u_k)_{k=1}^K$ for $K$-tuples of  functions in $\BV(\Omega)$ and  consider the admissible set
\begin{equation*}
	    \U \coloneqq \Bigl\{ \uu   \in \BV(\Omega)^K  \mid  \uu \geq 0  \wedge \sum_{k=1}^K u_k \leq 1 \Bigr\} \,.
\end{equation*}
Here and below $\uu \geq 0$ means that $u_k \geq 0$ for all $k \in \{1, \dots, K\}$.  Moreover, let $i_\U \colon \BV(\Omega)^K \to [0,\infty]$ denote the  associated  indicator function  taking the value $0$ inside $\U$ and the value $\infty$ outside of $\U$.  Throughout we use boldface  notation for various kinds  of $K$-tuples such as $\ff =(\phi_k)_{k=1}^K  \in  L^\infty(\Omega)^K$ or  $\ca = (a_k)_{k=1}^K \in \R^K$.

In this work we propose the following three-step segmentation procedure for solving Problem~\ref{pr:segment}  and segmenting the image $f \in \F$ into $K+1$ regions.

\begin{enumerate}[label=(N\arabic*),leftmargin=2.5em]
\item\label{M1}  \emph{Lifting:} Choose $K$ (feature enhancing) transforms  $\Phi_1, \dots, \Phi_K \colon \F \to L^\infty(\Om)$ in such a way that the intensity  values  of the $k$-th feature map $\phi_k \coloneqq \Phi_k(f)$ allow to well separate  region $\Sigma_k$ from the remaining part $\Omega \setminus \Sigma_k$.

\item\label{M2}  \emph{Minimization:}  For given parameter $\la>0$, compute a minimizer $(\uu^\la,\ca^\la,\cb^\la)\in  \BV(\Omega)^K  \times \R^{2K}  $ of  the energy functional $\func_{\ff; \la} \colon \BV(\Omega)^K  \times \R^{2K}  \to [0, \infty]$,
 \begin{equation} \label{eq:func}
    \func_{\ff; \la}(\uu,\ca,\cb) \coloneqq
    i_\U(\uu) + \lambda \sum_{k=1}^K   \abs{u_k}_{\rm TV}
    + \sum_{k=1}^K \int_{\Omega}(a_k - \phi_k)^2 u_k +  (b_k-\phi_k)^2(1-u_k) \,.
\end{equation}

\item\label{M3} \emph{Assignment:}  For each  $k \in \{0, \dots, K \}$ define $ \Sigma_k$ as the set of all $x \in \Omega$ such that $u^\la_k(x) $ is maximal among the values $u^\la_0(x), \dots, u^\la_K(x)$ where $u_0 \coloneqq 1-  \sum_{k=1}^K u_k $.
 \end{enumerate}

In some applications, the feature channels $\phi_1, \dots,  \phi_K$ may already be provided as the  input data without the need of additional filtering. For example, in an RGB image, each channel may represent specific regions well characterized by its color. Further examples of this type are  found in medical imaging, if for example,  one channel corresponds to a computed tomography (CT) image and another channel to magnetic resonance imaging  (MRI) data. In this case, bone structures are well revealed on the CT image, and soft tissue well on the MRI data.  Notice, however, that we are also interested in cases where channels are extracted via more general transforms $\Phi_1, \dots, \Phi_K$, for example,  by exploiting certain  expert knowledge and tailoring  the segmentation to specific  image features.

  \begin{figure}[htb!]
  \centering
      \includegraphics[width = 0.3\columnwidth,height = 0.3\columnwidth]{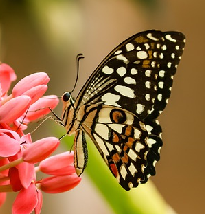}
  \includegraphics[width = 0.3\columnwidth,height = 0.3\columnwidth]{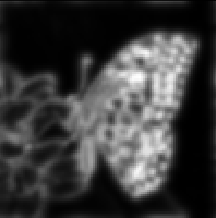}
    \includegraphics[width = 0.3\columnwidth,height = 0.3\columnwidth]{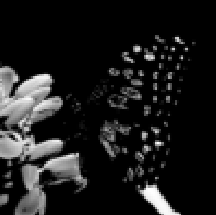}
    \\[0.2em]
    \phantom{\includegraphics[width = 0.3\columnwidth,height = 0.3\columnwidth]{butterfly.png}}
    \includegraphics[trim={1.2cm 1cm 0.3cm 0.5cm},clip,width = 0.3\columnwidth,height = 0.3\columnwidth]{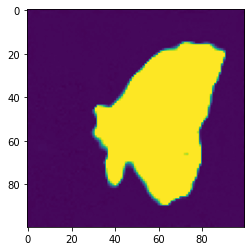}
     \includegraphics[trim={1.2cm 1cm 0.3cm 0.5cm},clip,width = 0.3\columnwidth,height = 0.3\columnwidth] {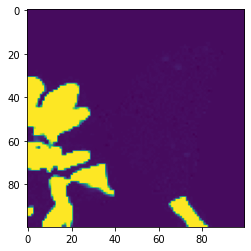}
\caption{\textbf{Illustration of the lifting approach.} Top: Given image $f$,  computed feature map $\phi_1$ (center) targeting the butterfly  and feature map $\phi_2$ (right) targeting the flower. The lifted version in this case  is obtained by color and Gabor filtering. Bottom: resulting minimizers  $u_1$, $u_2$ according to \eqref{eq:func}.}
\label{fig:motivation}
\end{figure}

To illustrate the reasoning behind such a strategy, consider the image in the upper left of Fig.~\ref{fig:motivation}.  The lifted images are obtained  by color  and Gabor filtering, with parameters provided in Table~\ref{tab:gab1}. Suppose we want to identify the butterfly and the flower, and think of the rest as background. Our framework then starts by computing a new set of images (two top right pictures), with the desired structures being highlighted, and then segments those structures via data fitting in the corresponding image channel. Our approach naturally leads to non-overlapping segmentation maps $u_k$, for which the constraint $\sum_{k=1}^K u_k \leq 1$ plays an important role. In particular, by relaxing the condition to $\sum_{k=1}^K u_k= 1$  it implies that no separate channel for the background is required. Segmentation procedure \ref{M1}-\ref{M3}  is universal and suitable for any multichannel or multiclass segmentation task.

\subsection{Related work}

Functional \eqref{eq:func} is inspired by  \citep{chan2006algorithms}, which introduces a  relaxation  of the Chan-Vese model for binary segmentation. Actually, for the special case $K=1$, \eqref{eq:func} reduces to the functional from \citep{chan2006algorithms}. One  main difference is the soft non-overlapping condition $\sum_{k=1}^K u_k \leq 1$, which is not required in the binary situation.  Moreover, the numerical minimization procedure we propose is very different from the one in \citep{chan2006algorithms}, which uses an explicit gradient procedure to solve a convex problem with fixed constants, together with occasionally updating constants. In addition, their analysis does not target stability and convergence with respect to the noise level.  Extensions of the Chan-Vese model for multiclass segmentation are presented  in~\citep{vese2002multiphase}.   In this work, however, $m$ overlapping level-set functions are used, resulting in $2^m$ distinct regions to be segmented. The resulting approach is fundamentally different from ours, as optimization is not performed over the segmentation masks $u_k$, but rather over the level set functions $\phi_k$.

Closely related, yet different approaches  that first lift the image into a higher dimensional space of feature maps before applying a segmentation algorithm have also been proposed. In \citep{zach2008fast, mevenkamp2016variational}, for example, segmentation functionals are used which are independent of the constants $\ca$ and $\cb$. This makes the final optimization less general but convex and easier to solve. In our case, due to the presence of unknown constants $\ca$ and $\cb$, the functional to be minimized is non-convex and requires the development of efficient algorithms tailored to handle it. In~\citep{kiechle2018model}, the authors propose a two-step learning approach, in which they learn filters to generate a pre-segmentation in a first step, and apply a modified Mumford-Shah functional in a second step.  We develop a numerical algorithm based on a non-convex primal-dual algorithm of   \citep{valkonen2014primal} for the second step.

Our proposed functional is non-convex because it contains constants $\ca$ and $\cb$. Treating these constants as being fixed would result in a convex problem to optimize. It is worth noting that there is also another source of non-convexity in this context, namely, actually working with a Mumford-Shah-like functional instead of using total-variation (TV) to obtain a convex relaxation. This is another challenging yet very exciting direction; see, for example, \citep{storath2014fast}, where an efficient algorithm for solving the Potts model has been developed. Combination and comparison with our multi-channel approach is an interesting line of research, but beyond the scope of this article.

The main contributions of the present paper are threefold. First, we introduce the lifting-based segmentation framework \ref{M1}-\ref{M3} using the energy functional  \eqref{eq:func}.  Second, we derive a mathematical analysis including stability and convergence with respect to data perturbations. Lastly, we present a numerical algorithm for which we conduct several experiments, in particular for applications in medical imaging.  From the mathematical point of view,  the  most important contributions consist in the presented stability and convergence analysis.  To the best of our knowledge, no such results have been previously derived in the literature, not even for the well studied binary segmentation task.  Existing theoretical analysis  of related models is usually concerned with  connections between the relaxed convex minimization problem, (for fixed values $\ca$, $\cb$) and the underlying non-convex problem.  In contrast, we study both stability and convergence properties of the relaxed functional.

\subsection{Outline}

The remaining  article  is structured as follows. In  Section~\ref{sec:analysis} we present the mathematical analysis of the proposed multiclass segmentation model  \eqref{eq:func}. In Section~\ref{sec4} we develop a numerical algorithm for actual numerical implementation. In Section~\ref{sec5}, we present  possible applications and experimental results. The paper concludes with a short summary and discussion given in Section~\ref{sec6}.

\section{Mathematical analysis} \label{sec:analysis}

We start  with the mathematical analysis  of  the proposed energy  functional  \eqref{eq:func}. Let $\Omega \subseteq \R^2$ be a bounded domain with Lipschitz boundary and denote by  $\BV(\Omega)$ the space of all functions $u \in L^1(\Omega)$ with finite (isotropic) total variation
\begin{equation*}
	 \abs{u}_{\rm TV}
	 \coloneqq
	 \abs{ \nabla u}_{2,1}
	 \coloneqq
	 \sup\Big\{ \int_{\Omega} u \div v \mid v \in\mathcal{C}_c^1(\Omega,\R^2) \wedge \abs{v}_{2,\infty}\leq 1\Big\}  \,,
\end{equation*}
where  $\abs{v}_{2,\infty}  \coloneqq  \sup_x (v_1(x)^2+ v_2(x)^2)^{1/2}$. Here  $\nabla u $ denotes  the distributional derivative  which for functions $u \in \BV(\Omega)$ is a vector-valued Radon measure having total mass $\abs{ \nabla u}_{2,1}$. It is well known that $\BV(\Omega)$ with norm $\norm{u}_{\rm TV} \coloneqq \norm{u}_{1}  + \abs{u}_{\rm TV}$ is a  Banach space.

\subsection{Notation and preliminaries} \label{sec:problem}

Let $\ff = (\phi_k)_{k=1}^K  \in L^\infty(\Omega)^K$ denote the available multichannel feature map and  $\la >0$ a regularization parameter.
Recall that the channels may either be already provided by the application or may be obtained by application of the feature transforms $\Phi_k$; see \ref{M1}. Further recall the admissible set $\U$ of (tuples of) segmentation maps, the corresponding indicator function $i_\U$ and the functional  $\func_{\ff;\la}$ defined in \eqref{eq:func}.
Using the notions
\begin{align*}
	\dist_{\ff}(\uu,\ca,\cb) &\coloneqq
	\sum_{k=1}^K 	\int_{\Omega}(a_k - \phi_k)^2 u_k
	 +  (b_k-\phi_k)^2(1-u_k) \,,
	\\
	\reg(\uu) &\coloneqq \sum_{k=1}^K \abs{u_k}_{\rm TV} \,,
\end{align*}
the considered energy functional, for $(\uu,\ca,\cb) \in \BV(\Omega)^K \times \R^{2K}$, takes the form
\begin{equation} \label{eq:func2}
\func_{\ff;\la}(\uu,\ca,\cb)
    =
    i_\U(\uu) + \lambda \reg(\uu)  +   \dist_{\ff}(\uu,\ca,\cb)   \,.
\end{equation}
Note that $\dom(\func_{\ff;\la}) \coloneqq \set{\uu \mid \func_{\ff;\la}(\uu) < \infty } = \U \times  \R^{2K}$ is nonempty, as the segmentation  map with constant channels  $u_k = 1/K$ is contained in $\U$.
Further, it is worth mentioning that the problem of minimizing~\eqref{eq:func} is bi-convex, meaning  it is convex in the variable $\uu \in \BV(\R)^K$ as well as in the variable $(\ca,\cb) \in \R^{2K}$. However it is not jointly convex in both variables.  Note that  the lower semicontinuity of the total variation  (for example, see \citep[p.~7]{giusti1984minimal}) implies that  $\reg(\uu) \leq \liminf_{n \to \infty} \reg(\uu_n) $ for any sequence  $ (\uu_n)_{n\in \N} $ with $L^1$-limit $\uu$.  Further,  at several  places of our analysis we make use  of the  following compactness result.

 \begin{lem}\label{lem:rellich}
If  $(\uu_n)_{n\in\N} \in (\BV(\Omega)^K)^\N $ satisfies $i_\U(\uu_n) +  \reg(\uu_n) \leq C $  for some constant $C \in  (0,\infty)$,  then $\norm{\uu_{\tau(k)} - \uu}_{1} \to 0$ as $k\to\infty$ for some subsequence $\tau$ and  some $\uu\in \BV(\Omega)^K$.
\end{lem}

\begin{proof}
If  $i_\U(\uu_n) + \reg(\uu_n) \leq C $,  then it holds  $0 \leq u_{k,n} \leq 1$ and  $\abs{u_{k,n}}_{\rm TV} \leq C$ for all $k\in \{1, \dots, K\}$ and $n \in \N$. Therefore the sequences $(\uu_{k,n})_{n\in \N}$  are bounded in $\BV(\Omega)$. The compactness theorem for $\BV(\Omega)$  (see, for example, \citep[p.~132]{ambrosio2000functions} or \citep[p.~176]{evans1992measure}) thus implies the existence of a subsequence $\tau$ and  some $\uu\in \BV(\Omega)^K$ with $\norm{\uu_{\tau(k)} - \uu}_{1} \to 0$.
\end{proof}

\subsection{Reduced formulation} \label{sec:equivalence}

Functional \eqref{eq:func2} is  jointly minimized over $\uu \in \BV(\Omega)^K$ and $ (\ca,\cb)  \in \R^{2K}$. Throughout this paper we will make  use of an equivalent  reduced optimization problem in the variable $\uu$  only, by explicitly computing the minimizers  in the variable $(\ca,\cb)$. In that context, we note  the following elementary result.

\begin{lem}\label{lem:partial}
For all $ \uu \in \U$, the set $\argmin \func_{\ff;\la}(\uu,\cdot)$ is non-empty, equals $\argmin \dist_{\ff}(\uu,\cdot)$,  and  given by
\begin{equation}\label{eq:argminset} 
	\averageC(\uu, \ff) \coloneqq \prod_{k=1}^K  \bigl( \averageA(u_k, \phi_k) ,  \averageA(1-u_k, \phi_k) \bigr)
\end{equation}
where
 \begin{align}\label{eq:means}
   \averageA(u_k, \phi_k) & \coloneqq
   \begin{cases}
   	\R & \text{ if $u_k = 0$ }
	\\
	\left\{ \frac{1}{ \norm{u_k}_{1}} \int_{\Omega} \phi_k u_k  \right\} & \text{ otherwise } \,.
	\end{cases}
\end{align}
\end{lem}

\begin{proof}
Clearly, $\argmin \func_{\ff;\la}(\uu,\cdot) = \argmin \dist_{\ff}(\uu,\cdot)$ and minimizing   $\dist_{\ff}(\uu,\cdot)$ is separable  in  the components $a_k, b_k$.  Hence minimizers are found by separately minimizing $\int_{\Omega}(a_k - \phi_k)^2 u_k$ with respect to $a_k$ and   $\int_{\Omega}(b_k - \phi_k)^2 (1-u_k)$ with respect to $b_k$. We have   $\int_{\Omega}(a_k - \phi_k)^2 u_k  = a_k^2 \norm{u_k}_1   - 2 a_k \int_{\Omega} \phi_ku_k + \int_{\Omega} \phi_k^2 u_k $ and for $u_k \neq 0$ minimizers are given by $\int_{\Omega} \phi_k u_k / \norm{u_k}_1$. If $u_k=0$  any $a_k \in \R$ is a minimizer of $\dist_{\ff}(\uu,\cdot)$. Similar arguments for the second minimization problem complete  the proof. \end{proof}

Based on Lemma~\ref{lem:partial} we define the reduced energy functional $\funcR_{\ff,\la} \colon \BV(\Omega)^K \to [0, \infty] $ by
\begin{equation}  \label{eq:func3}
	\funcR_{\ff,\la}(\uu)
	 \coloneqq
	 \begin{cases}
	 \lambda \reg(\uu) +   \inf \dist_{\ff}(\uu, \cdot)
	 &\text{ if } \uu \in \U \,,\\
	 \infty  & \text{ if } \uu \not \in \U \,.
\end{cases}
\end{equation}
According to Lemma~\ref{lem:partial}, the infimum $  \inf \dist_{\ff}(\uu, \cdot) $ is  attained   for all $\uu \in \U$ and the corresponding set of minimizers is given by \eqref{eq:means}. For the following lemma note that $\argmin \funcR_{\ff,\la} \subseteq \U$.

\begin{lem}[Equivalence] \label{lem:equiv}
For all $(\uu^\la,\ca^\la,\cb^\la) \in \BV(\Omega)^K \times \R^{2K}$ the following statements  are equivalent:
\begin{enumerate}[label=(\alph*)]
\item\label{eq1}   $ (\uu^\la,\ca^\la,\cb^\la) \in \argmin \func_{\ff,\la}$.
\item\label{eq2}   $ \uu^\la \in\argmin \funcR_{\ff,\la}$ and   {$(\ca^\la,\cb^\la)  \in  \argmin    \dist_{\ff}(\uu^{{\la}}, \cdot) $.}
\end{enumerate}
\end{lem}

\begin{proof}
If $ (\uu^\la,\ca^\la,\cb^\la)$ is a minimizer of $ \func_{\ff,\la}$ then clearly $(\ca^\la,\cb^\la) \in \argmin \func_{\ff;\la}(\uu^\la,\cdot)$ and  for  all $\uu \in \BV(\Om)^K$ we have $ \funcR_{\ff,\la} (\uu) \geq  \func_{\ff,\la} (\uu^\la,\ca^\la,\cb^\la)  =  \funcR_{\ff,\la} (\uu^\la)$.
Conversely, if $\uu^\la$ minimizes $\funcR_{\ff,\la}$ and $(\ca^\la,\cb^\la)  \in  \argmin   \func_{\ff;\la}(\uu^\la,\cdot)$, then for  $(\uu,\ca,\cb) \in \BV(\Om)^K \times \R^{2K}$ we have
$ \func_{\ff,\la} (\uu,\ca,\cb)  \geq   \funcR_{\ff,\la} (\uu) \geq   \funcR_{\ff,\la} (\uu^\la) =  \func (\uu^\la,\ca^\la,\cb^\la)$ which  completes the proof.
\end{proof}

In the following we prove the existence and stability of minimizers of  $\funcR_{\ff,\la}$ which according to Lemma \ref{lem:equiv} is equivalent to existence and stability of  minimizers of  $\func_{\ff;\la}$.  We further investigate the convergence of minimizers as the error in the data $\ff$ tends to zero.

\subsection{Existence}\label{sec3}

We start with the existence of minimizers of  $\funcR_{\ff,\la}$.

\begin{thm}[Existence]
For all $\ff \in L^\infty(\Omega)^K$ and $\la >0$, functional $\funcR_{\ff,\la}$ admits at least one global minimizer. \label{thm:existence}
\end{thm}

\begin{proof}
Since $ (1, \dots, 1)/K  \in \U$, the domain of $\funcR_{\ff,\la}$ is non-empty and we can choose a sequence $(\uu_n)_{n\in\N} \in \U^\N$  such that $\funcR_{\ff,\la} (\uu_n) \to \inf \funcR_{\ff,\la}$. In particular $(i_\U(\uu_n) + \reg(\uu_n))_{n\in \N}$ is bounded and by Lemma~\ref{lem:rellich} there exists a subsequence, again denoted by $(\uu_n)_{n \in \N}$, that converges in $L^1(\Omega)^K$ to   some $\uu^\la \in L^\infty(\Omega)$. By moving to another subsequence, point-wise convergence can be established almost everywhere, from which the conclusion $\uu^\la \in \U$ can be drawn.   {Next, select $(\ca_n,\cb_n) \in \R^{2K}$ with  $\funcR_{\ff,\la}(\uu_n) = \func_{\ff,\la}(\uu_n,\ca_n,\cb_n)$ according to  \eqref{eq:argminset}, \eqref{eq:means}.} {In particular, for $u_k \neq 0$ and  $u_k \neq 1$ we have $a_{k,n} =  \norm{u_k}_{1}^{-1} \int_{\Omega} \phi_k u_k$ and $b_{k,n} = \norm{1-u_k}_{1}^{-1} \int_{\Omega} \phi_k (1-u_k)$, respectively.}  Thus if  $\uu^\la_k \neq 0,1$ for all $k \in \{1,\dots, K\}$ we can assume $u_{k,n} \neq 0,1$ for all $k,n$ and by  \eqref{eq:means} we have
\begin{align*}
0 \leq  a_{k,n} &=  \frac{\inner{u_{k,n}, \phi_k} }{\norm{u_{k,n}}_1}   \leq \norm{\phi_k}_\infty \\
0 \leq  b_{k,n} &= \frac{\inner{1-u_{k,n},\phi_k }}{\norm{1-u_{k,n}}_1}  \leq \norm{\phi_k}_\infty \,.
 \end{align*}
If either $u_k^\la = 0$ or $u_k^\la = 1$, the above identities for $a_{k,n}, b_{k,n} $ may not hold. According to Lemma~\ref{lem:partial}, in this case the numbers $a_{k,n}, b_{k,n} $ can be chosen arbitrarily, and we can therefore still assume that $0 \leq a_{k,n}, b_{k,n} \leq \norm{\phi_k}_\infty$. Up to extracting another subsequence, in any case we can assume $(\ca_n,  \cb_n)_{n\in \N}  \to  (\ca^\la, \cb^\la) \in [0, \infty)^{2K}$. Fatou's Lemma yields
\begin{align*}
    \dist_{\ff}(\uu^\la, \ca^\la, \cb^\la)
     &= \sum_{k=1}^K 	\int_{\Omega}(a_k^\la - \phi_k)^2 u_k^\la  +  (b_k^\la-\phi_k)^2(1-u_k^\la)
      \\
    & \leq
     \sum_{k=1}^K \liminf_{n\to \infty} 	
     \int_{\Omega}(a_{k,n} - \phi_k)^2 u_{k,n}  +  (b_{k,n}-\phi_k)^2(1-u_{k,n})
      \\
    & \leq  \liminf_{n\to \infty}\dist_{\ff}(\uu_n, \ca_n, \cb_n) \,.
    \end{align*}
Together with the lower semi-continuity of $\reg$, we obtain
\begin{align*}
\funcR_{\ff,\la}(\uu^\la)
    &= i_\U(\uu^\la) + \lambda \reg(\uu^\la) +    \dist_{\ff}(\uu^\la, \ca^\la,\cb^\la)\\
    &\leq i_\U(\uu^\la) + \la    \liminf_{n\to \infty}  \reg(\uu_n)  +    \liminf_{n\to \infty} \dist_{\ff}(\uu_n,\ca_n,\cb_n)\\
  &\leq  \liminf_{n\to \infty} i_\U(\uu_n) + \la    \reg(\uu_n)  +    \dist_{\ff}(\uu_n,\ca_n,\cb_n)\\
  &\leq  \lim_{n\to \infty} i_\U(\uu_n) + \la   \reg(\uu_n) +    \inf_{\ca,\cb} \dist_{\ff}(\uu_n, \ca,\cb)
  \\
 &=   \inf \funcR_{\ff,\la}
\end{align*}
Hence $\uu^\la$ is a minimizer of $\funcR_{\ff,\la}$.
\end{proof}

\subsection{Stability}

Next we investigate the stability of  minimizers of $\funcR_{\ff,\la}$ with respect to data $\ff$.

\begin{thm}[Stability] \label{thm:stability}
Let $\ff, \ff_n  \in  L^\infty(\Omega)^K$ for $n \in \N$ such that  $\norm{ \ff_n- \ff }_{\infty} \to 0$ and take $\uu_n \in \argmin \funcR_{\phi_n,\la}$. Then $(\uu_n)_{n\in\N}$ has at least one  $L^1$-norm convergent subsequence and the limit of  any $L^1$-norm convergent subsequence  $(\uu_{\tau(n)})_{n\in\N}$ is a minimizer $\uu^\la$ of $\funcR_{\ff,\la}$ that satisfies  $ \reg(\uu_{\tau(n)})\to \reg(\uu^\la)$ as $n \to \infty$.
\end{thm}

\begin{proof}
Because $\uu_n \in\argmin \funcR_{\phi_n,\la}$,  we have that  $(\funcR_{\ff_n,\la}(\uu_n))_{n\in \N}$ is bounded and therefore $(\reg(\uu_n))_{n \in \N}$  is bounded, too. By Theorem~\ref{lem:rellich} there exists a subsequence, that we again denote  by $(\uu_n)_{n \in \N}$,  that converges in $L^1(\Omega)^K$ to some $\uu^\la \in \BV(\Omega)^K$.  Like in the proof  of Theorem~\ref{thm:existence}, we  {construct $(\ca_n, \cb_n)$  such that  $(\uu_n, \ca_n, \cb_n) \in\argmin \func_{\ff_n,\la}$ and $0  \leq a_{k,n}, b_{k,n} \leq \norm{\phi_{k,n}}_\infty$. As $\norm{\ff_n- \ff}_{\infty} \to 0$,   the elements $\ff_n$ are uniformly bounded and thus} $(a_{k,n})_{n\in \N}$ and $(b_{k,n})_{n\in \N}$ are bounded. Therefore, up to the  extraction of another subsequence, we can assume  $(\ca_n, \cb_n) \to (\ca^\la, \cb^\la)$.

It remains to show that $\uu^\la$ is a minimizer of $\func_{\ff;\la}$ and that $ \reg(\uu_n)\to \reg(\uu^\la)$. {After passing to a subsequence that converges pointwise almost everywhere, from}  Fatou's lemma and the semi-continuity of $\reg$  one derives   $\funcR_{\ff,\la}(\uu^\la)   \leq \liminf_{n\to \infty} \func_{\ff_n,\la}(\uu_n, \ca_n, \cb_n)$. Therefore, for all $(\uu, \ca,\cb) \in \BV(\Om)^K \times \R^{2K}$, we have
\begin{align*}
    \funcR_{\ff,\la}(\uu^\la)
        & \leq \liminf_{n\to \infty}  i_\U(\uu_n) + \la \reg(\uu_n) +    \dist_{\phi_n}(\uu_n,\ca_n,\cb_n)
    \\
    & \leq \limsup_{n\to \infty}  i_\U(\uu_n) + \la  \reg(\uu_n)  +    \dist_{\phi_n}(\uu_n,\ca_n,\cb_n)
    \\
    & \leq \lim_{n\to \infty}  i_\U(\uu) + \la  \reg(\uu) +   \dist_{\phi_n}(\uu, \ca,\cb)
    \\
    & =  i_\U(\uu) + \la  \reg(\uu) +   \dist_{\phi}(\uu, \ca,\cb) =  \func_{\ff,\la}(\uu, \ca, \cb) \,.
 \end{align*}
Taking   the infimum over $(\ca,\cb)$ shows  $\funcR_{\ff,\la}(\uu^\la) \leq \funcR_{\ff,\la}(\uu)$ which  implies  $\uu^\la \in \argmin \funcR_{\ff,\la}$. Choosing $\uu =  \uu^\la$  in the last displayed equation shows  $\funcR_{\ff,\la}(\uu^\la) = \lim_{n\to \infty}
 \la  \reg(\uu_n)+    \dist_{\phi_n}(\uu_n,\ca_n,\cb_n)$ and
\begin{multline*}
\limsup_{n \to \infty } \la \reg(\uu_n)
 \leq  \lim_{n\to \infty}
 \la  \reg(\uu_n)+    \dist_{\phi_n}(\uu_n,\ca_n,\cb_n)
  -  \liminf_{n \to \infty }   \dist_{\phi_n}(\uu_n,\ca_n,\cb_n)
\\  \leq
 \funcR_{\ff,\la}(\uu^\la)
 -     \dist_{\ff}(\uu^\la, \ca^\la,\ca^\la) = \la  \reg(\uu^\la) \,,
 \end{multline*}
which  concludes the proof.
\end{proof}

\subsection{Convergence}

In the following we investigate the convergence of minimizers of  $\funcR_{\ff,\la}$ to minimizers of the  following limiting constraint optimization  problem
 \begin{equation} \label{eq:lim}
\left\{
\begin{aligned}
	&\min_{\uu} \reg(\uu)
	\\
	&\text{ s.t. }
	\uu \in \argmin i_\U   +   \inf_{\ca,\cb} \dist_{\ff}(\cdot, \ca,\cb)
\,.
\end{aligned}
\right.
\end{equation}
In general, $ \argmin  i_\U   +   \inf_{\ca,\cb} \dist_{{{\phi}}}(\cdot, \ca,\cb) $ can  be empty, reflecting the ill-posedness of  \eqref{eq:lim} and the need for the relaxed problem \eqref{eq:func2}. However, for certain  $\ff$, which we refer to exact data,  such minimizers  exist and define a  segmentation function that is independent of any parameters. When minimizing the energy functional $\funcR_{\ff,\la}$,  we then interpret $\ff$ as perturbed version  of  the exact data, and $\la$ as regularization parameter accounting for stability. In the convergence analysis, we  study convergence of minimizers of $ \funcR_{\ff,\la}$ to solutions of \eqref{eq:lim} as the noise level   and the regularization parameter tend to zero.  We are not aware of any such an analysis in the segmentation literature, but think that this approach  can  be the starting point of  new insights  and algorithms.

\begin{rem}
Denote by  $\mathbbm{1}_\Sigma$ the indicator  function of a set $\Sigma \subseteq\Om$ taking  the value $1$ inside $\Sigma$ and $0$ outside.  In order to focus on the main ideas, in the following we  consider exact feature maps $\ff^* \in L^\infty(\Omega)^K$ with binary channels
\begin{equation} \label{eq:bi}
	\phi_k^* = a_k^* \mathbbm{1}_{\Sigma_k} + b_k^* \mathbbm{1}_{\Omega \setminus \Sigma_k} \,,
\end{equation}
where
\begin{itemize}
\item $ \Sigma_k \subseteq \Omega$ are pairwise disjoint  with finite perimeters
\item $(\ca^*, \cb^*) \in \R^{2K}$ satisfies $\forall k \colon a_k^*  \neq b_k^*$.
\end{itemize}
Recall that the perimeter $\peri(\Sigma) \coloneqq \abs{\mathbbm{1}_\Sigma}_{\rm TV}$ of a set $\Sigma$ is defined as the total variation of the indicator function $\mathbbm{1}_\Sigma$. In particular, functions  of the form \eqref{eq:bi} satisfy $\abs{\phi_k}_{\rm TV} = \abs{b_k^* - a_k^*} \peri(\Sigma_k)$.
\end{rem}

\begin{lem} \label{lem:exact}
For any exact data $\ff^* \in L^\infty(\Omega)^K$  of the form \eqref{eq:bi}, the  set  of solutions of \eqref{eq:lim} is non-empty  and consists  of all  $ \uu^\star \in \BV(\Om)^K$ with $u_k^\star  \in  \{ \mathbbm{1}_{\Sigma_k}, \mathbbm{1}_{\Omega \setminus \Sigma_k} \}$.
\end{lem}

\begin{proof}
Let $u_k^\star  \in  \{ \mathbbm{1}_{\Sigma_k}, \mathbbm{1}_{\Omega \setminus \Sigma_k} \}$. Clearly, $0 \leq \funcR_{\ff^*,0}$ and $0 = \funcR_{\ff^*,0}(\uu^\star)$. Therefore $\uu^\star \in \argmin i_\U   +   \inf_{\ca,\cb} \dist_{\ff}(\cdot, \ca,\cb)$. Moreover,  $\reg(\uu^\star) = \sum_{k=1}^K  \peri(\Sigma_{k})$ which shows that $\uu^*$ solves \eqref{eq:lim}.
 \end{proof}

We now have the following convergence result.

\begin{thm}[Convergence]\label{th-cv}
Let   $\ff^* \in L^\infty(\Omega)^K$ be exact data of the form \eqref{eq:bi}, $(\delta_n)_{n\in\N}\in (0, \infty)^\N$  converge to zero and $(\ff_n)_{n\in\N} \in L^\infty(\Om)^K$ be noisy data with $\norm{\ff^*-\ff_n}_{\infty}\leq\delta_n$.   Further, let  $\la_n \in (0,\infty)^\N$ be a sequence  of regularization parameters with
$\la_n \to 0$ and $\delta_n^2/\la_n \to 0$ and let  $\uu_n \in \argmin  \funcR_{\ff_n, \la_n}  $. Then  $(\uu_n)_{n\in \N}$ has at least  one $L^1$-convergent subsequence. Moreover, the limit of any $L^1$-convergent subsequence $(\uu_{\tau(n)})_{n\in \N}$ is a solution $\uu^*$ of \eqref{eq:lim} with
$\reg(\uu_{\tau(n)}) \to  \reg(\uu^*)$.
\end{thm}

\begin{proof}
According to Lemma \ref{lem:exact}, $\uu^* = (\mathbbm{1}_{\Sigma_k})_{k\in \N}$ is a solution of \eqref{eq:lim}. Further, we  have
\begin{equation*}
 	\dist_{\ff_n}(\uu^*, \ca^*, \cb^*)
	=  \sum_{k=1}^K
	\int_{\Sigma_k} {(\phi_{k,n} -  a_k^* )^{2}} + \int_{\Omega \setminus \Sigma_k} (\phi_{k,n} -  b_k^*)^2
	=  \sum_{k=1}^K  \int_{\Omega} (\phi_{k,n} - \phi_k^*)^2  \leq K \abs{\Om} \delta_n^2 \,.
\end{equation*}
Together with  the definition of  $\uu_n$ and choosing elements $(\ca_n,\cb_n) \in \argmin  \dist_{\phi_n}(\uu_n, \cdot)$ this shows
\begin{equation*}
	 \dist_{\phi_n}(\uu_n,\ca_n,\cb_n) + \la_n \reg(\uu_n) \leq  K \abs{\Om} \delta_n^2 +  \la_n \reg(\uu^*)  \,.
\end{equation*}
Together with the parameter choice $\la_n, \delta_n^2/\la_n \to 0$ and the continuity  of $\dist_{\ff_n}(\uu_n,\ca_n,\cb_n)$ in $\ff_n$ this  shows
\begin{align} \label{eq:conv1}
	\lim_{n \to \infty}\dist_{\ff_n}(\uu_n,\ca_n,\cb_n) =0
	\\ \label{eq:conv2}
	\limsup_{n \to \infty} \reg(\uu_n) \leq  \reg(\uu^*) \,.
\end{align}
Lemma~\ref{lem:rellich} and identity \eqref{eq:conv2}  imply the existence of an $L^1$-convergent subsequence $(\uu_{\tau(n)})_{n\in \N}$.   According to  \eqref{eq:conv1}, the limit  of any such  subsequence  is a solution of \eqref{eq:lim} and without loss of generality we can assume $\uu_{\tau(n)} \to \uu^*$. From \eqref{eq:conv2} and the lower semi-continuity of $\reg$ we finally derive $\reg(\uu_{\tau(n)}) \to \reg(\uu^*)$.
\end{proof}

\section{Algorithm development}\label{sec4}

In this section, we  establish a numerical algorithm for minimizing the reduced segmentation functional $\funcR_{f, \la}$. For that purpose, we will  first   introduce a discrete framework for  which we actually present our optimization procedure.

\subsection{Discretization} \label{discretesetting}

In the following we work with discrete images and feature maps in  $\Hr \coloneqq \R^{N_1 \times N_2}$, which is a finite dimensional Hilbert space with  inner  product $  \inner{u,v} = \sum_{i} u[i] v[i]$ for $u, v \in \Hr$ where $i = (i_1, i_2) \in \set{1, \dots, N_1} \times \set{1, \dots, N_2}$.  Moreover we introduce the following discrete   counterparts of ingredients of the continuous functional~\eqref{eq:func3}.

\begin{itemize}
\item The discrete gradient  $\grad = (\grad_1, \grad_2)  \colon \Hr \to \Hr \times \Hr$ is defined by  forward  differences with Neumann boundary conditions
\begin{align*}
    (\grad_1 u)[i] &\coloneqq
    \begin{cases}
    	(u[i_1+1,i_2]-u[i_1,i_2]) / h & \text{if }  i_1<N_1\\
    	0 &  \text{if } i_1 = N_1
    \end{cases}
    \\[0.2em]
    (\grad_2 u)[i] &\coloneqq
    \begin{cases} (u[i_1,i_2+1]-u[i_1,i_2])/h & \text{if }  i_2 < N_2 \\
	0 &  \text{if }  i_2=N_2  \,.\end{cases}
\end{align*}
Its adjoint  is given by $\grad^* (v_1, v_2) = \grad^*_1 v_1 + \grad_2^* v_2 =: -\div (v_1, v_2)$  where $\div \colon \Hr \times \Hr \to \Hr$ is the  discrete  divergence operator and for  $(v_1, v_2) \in \Hr \times \Hr$ we have
\begin{align*}
    (\grad^*_1 v_1)[i] & =
    \begin{cases}
    	-(v_1[i_1,i_2] - v_1[i_1-1,i_2]) / h & \text{if }  1<i_1<N_1\\
    	-v_1[1,i_2] &  \text{if } i_1 = 1 \\
    	\phantom{-}v_1[N_1-1,i_2] &  \text{if } i_1 = N_1
    \end{cases}
    \\[0.1em]
    (\grad^*_2 v_2)[i] &=
    \begin{cases}
    - (v_2[i_1,i_2] - v_2[i_1,i_2-1])/h & \text{if } 1< i_2 < N_2 \\
	 -v_2[i_1,1] &  \text{if } i_2 = 1 \\
    	\phantom{-}v_2[i_1,N_2-1] &  \text{if } i_2 = N_2  \,.
	\end{cases}
\end{align*}

Finally we write $\gradd \uu  \coloneqq (\grad u_1, \dots, \grad u_K)$ for the discrete gradient applied componentwise to $\uu \in  \Hr^K$.
\item
The discrete (isotropic) TV semi-norm of some image $u \in \Hr$ is defined as
\begin{equation*}
    \norm{ \grad_1 u}_{1,2} \coloneqq \sum_{i} \sqrt{(\grad_1 u[i])^2 + (\grad_2 u[i])^2} \,.
\end{equation*}
We write  $\U = \set{\uu \in \Hr^K \mid \uu \geq 0 \wedge \sum_k u_k =1}$ for the discrete admissible set and $i_\U$ for the corresponding indicator function.

\item
Let  $\mathbbm{1} \in \Hr$ denotes the all-ones image, $\ff = (\phi_k)_{k=1}^K \in \Hr^K$ the discrete feature  map  with $K$-channels to be segmented  and $\uu = (u_k)_{k=1}^K\in \Hr^K$ the desired segmentation function.
The discretization of the data fitting term in the reduced  energy is written  as   $\norm{ \diffC(\uu)}_{1} $  where
\begin{align} \label{eq:dd1}
    & \diffC(\uu)  \coloneqq  (\diffB(u_k), \diffB(1-u_k))_{k=1}^K  \\ \label{eq:dd2}
    &\diffB(u_k)  \coloneqq  u_k \cdot (A(u_k) \mathbbm{1} - \phi_k)^2   \\ \label{eq:dd3}
    &A(u) \coloneqq   \inner{u_k,\phi_k} / \abs{u_k}_{1,\epsilon}    \,.
    \end{align}
Here  $\norm{ \cdot}_{1}$ is the $\ell_1$-norm applied to  $\diffC(\uu) \in \Hr^{2K}$ and the averages $A(\cdot)$ use the approximation  $\abs{u}_{1,\epsilon}  \coloneqq  \sum_{i} (u[i]^2+\eps)^{1/2}$ of the $\ell_1$-norm, making it single-valued and twice differentiable.
\end{itemize}

Using the notions given above, the discrete reduced energy functional reads
\begin{equation} \label{eq:minD}
	\funcR_{\ff,\la}(\uu)
	\coloneqq i_\U (\uu) +  \lambda \sum_{k=1}^K   \norm{ \grad u_k}_{1,2} + \norm{ \diffC(\uu)}_{1,1} \,.
\end{equation}
Optimization problem \eqref{eq:minD} is a non-convex,  non-smooth  and   a  challenging large scale problem to be solved. Its particular  structure, however, allows to apply various splitting type optimization algorithms.  In particular,  we will demonstrate that  it can be solved with the algorithm of   \citep{valkonen2014primal}, which itself is a generalization of the Chambolle-Pock algorithm \citep{chambolle2011first} for convex problems.

\subsection{Nonlinear primal dual algorithm}

Our algorithm is a particular instance of the primal dual hybrid gradient (PDHG) algorithm of \citep{valkonen2014primal}, which is a generic algorithmic framework solving minimization problems of the following composite form
\begin{equation}\label{primal}
    \min_{\xx \in \X} \pdF( \pdK (\xx) ) + \pdG (\xx) \,,
\end{equation}
where $\pdK \colon \X \to \Y$ is a possibly  nonlinear mapping between Hilbert spaces $\X$, $\Y$ and  $\pdF \colon \Y \to [0,\infty]$ and $\pdG \colon \X \to [0,\infty]$ are convex and lower semi-continuous functionals.  Essential components of the PDHG algorithm   are the proximal operator $\prox_{\pdH} \colon \X \to \X$ and the Fenchel conjugate $ \pdH^* \colon \X \to [0, \infty]$, respectively,  associated  to a given functional  $\pdH \colon \X \to [0, \infty]$, defined by
\begin{align*} 
    \prox_{\pdH}(\xx) &= \argmin_{z \in \X} \pdH(z)  + \frac{1}{2}  \norm{z-\xx}^2
    \\ 
    \pdH^*(\xx)  &=  \sup_{z \in \X} \inner{\xx,z} - \pdH(z) \,.
\end{align*}
Note that the  nonlinear PDHG algorithm  itself is an extension of primal dual optimization scheme  proposed in~\citep{chambolle2011first} from linear  to nonlinear operators.

The PDHG algorithm for \eqref{primal} with parameters $\sigma, \tau > 0$ and  $\theta\in[0,1]$   generates a sequence  $(\xx^n, \bar \xx^n, \yy^n)_{n\in \N} \in (\X \times \X \times \Y)^\N$ by
\begin{align} \label{eq:pdhg1}
\yy^{n+1} &\coloneqq \prox_{\sigma\mathcal{F}^*}(\yy^n + \sigma \pdK (\bar \xx^n) )
\\ \label{eq:pdhg2}
\xx^{n+1} &\coloneqq \prox_{\tau\pdG}(\xx^n -  \tau \pdK'(\xx^n)^* \yy^n )
\\ \label{eq:pdhg3}
\bar{\xx}^{n+1} &\coloneqq  \xx^{n+1} + \theta(\xx^{n+1} - \xx^n ) \,,
\end{align}
for some initializations  $\xx^0, \bar{\xx}^0 \in \X$ and $ \yy^0 \in \Y$. Here $\pdK'(\xx)^*$ is the adjoint of the derivative  of $\pdK$ at $\xx$. Convergence of \eqref{eq:pdhg1}-\eqref{eq:pdhg3} to saddle points of \eqref{primal} has been analyzed in \citep{valkonen2014primal}.

\subsection{Derivation of algorithm}

Inspecting the functional  \eqref{eq:minD} to be minimized and the PDHG algorithm \eqref{eq:pdhg1}-\eqref{eq:pdhg3}  for the generic problem  \eqref{primal}, we  see that it can be applied to our setting with
\begin{align*}
\X  &= \Hr^K \\
 \Y &=  \Hr^{2K}  \times \Hr^{2K}  \\
\pdK  &=  (\gradd,  \diffC)    \\
\pdF ( \vv, \ww)    &=  \lambda \norm{\vv}_{2,1} +\norm{\ww}_{1,1} \\
\pdG &=   i_\U   \,.
\end{align*}
Recall that $\diffC$ defined by \eqref{eq:dd1}-\eqref{eq:dd3}  is nonlinear and  therefore  $\pdK$ is nonlinear,  too. Further, the functionals $\pdF$ and $\pdG$ are convex and lower semi-continuous.   The actual practical  implementation  requires computing proximal mappings, derivatives, and Fenchel conjugates.  This will be done in the following.

\begin{itemize}

\item \emph{Fenchel conjugate of $\pdF$:}
The  functional  $\pdF$ is separable in $\vv$ and $\ww$, and therefore  its  Fenchel conjugate is a sum of the  Fenchel conjugates of  $\norm{\cdot}_{2,1} $ and $\norm{\cdot }_{1,1}$, which are  both well known. Actually, they are given by indicator functions of the unit ball of the dual norms  and therefore
\begin{equation*}
\pdF^*(\vv,\ww) =   i_{2,\infty}(\vv/\la) + i_{\infty,\infty}(\ww) \,.
\end{equation*}
Here $i_{q,\infty} = \norm{\cdot}_{q,\infty}^*$ denotes the indicator function of the unit ball $\{ \vv \mid \norm{\vv}_{q,\infty} \leq 1\} \subseteq  (\Hr^K)^2$ for the norm  $\norm{\cdot}_{q,\infty}$ for $q=2,\infty$.

\item \emph{Proximal operators of $\pdF^*$, $\pdG$:}
Let us next compute the proximal operator of $\pdF^*(\vv, \ww) =   i_{2,\infty}(\vv/\la) + i_{\infty,\infty}(\ww)$ which  is the separable sum of two  indicator functions. The  proximal operator  of the indicator function $i_S$ of some set $S$ is   known to be given as the projection on  $S$.
The projections   $P_{\infty,p} \colon  (\Hr^K)^2 \to  (\Hr^K)^2$ onto the unit ball in the  $(\infty,q)$-norm  for $q = 2, \infty$ can easily be computed and given by
 \begin{align*}
        &(P_{2, \infty} (\vv))[i,k]= \frac{v[i,k]}{\max\{1, (v_1[i,k]^2 + v_2[i,k])^{1/2}\}} \\
        &(P_{\infty, \infty} (p))[i,k] =  \Bigl( \frac{v_j[i,k]}{\max\{1, v_j[i,k] \}} \Bigr)_{j=1,2}  \,.
\end{align*}
Thus  the proximal operator  of $\pdF^*$ is  given  by
\begin{equation*}
\prox_{\pdF^*}(\vv, \ww) = (P_{2, \infty} (\vv/\la), P_{\infty, \infty} (\ww))  \,.
\end{equation*}
The proximal operator of $\pdG = i_\U $ is the orthogonal projection $P_\U$ onto the simplex $\U$. In this case,  no explicit formula is available. However,  the projection  can be computed with a finite number of steps, and for our implementation we use the algorithm proposed in~\citep{michelot1986finite} for that purpose.

\item \emph{Derivative computation:}
Finally we compute  the adjoint  of the derivative of $\pdK = (\gradd, \diffD)$. The discrete gradient operator   is linear  and therefore $\gradd'(\uu) = \gradd $. Further from \eqref{eq:dd1}-\eqref{eq:dd3}  we see that computing the derivative  of $\pdK$ amounts
 to  computing the derivative of
\begin{equation*}
\diffB(u) = u \cdot \left(  \frac{\inner{u,\phi} }{ \abs{u}_{1,\epsilon}} \mathbbm{1} - \phi \right)^2  \,.
\end{equation*}
For that purpose we  write $\phi, u \in \Hr$ as column vectors,
\begin{align*}
	u &= (u_1, \dots, u_N )^\intercal  \in \R^N\\
	\phi  &= (\phi_1, \dots, \phi_N )^\intercal  \in \R^N
\end{align*}
with $N  \coloneqq N_1, N_2$.
By the chain rule we find that the derivative  $(D\diffB)(u)$ and its adjoint are given by the following $N \times N$ matrices
\begin{align}\label{derivative}
(D\diffB)(u) &= T (u)^2  -  2 U(u) T(u) Q(u)   \\
(D\diffB)(u)^* &= T (u)^2  -  2 Q(u)^\intercal T(u) U(u)
\end{align}
where
\begin{align*}
   T(u) &\coloneqq \diag \bigl( \inner{u,\phi} \abs{u}_{1,\epsilon}^{-1}   \mathbbm{1} - \phi \bigr)
\\
U(u)&\coloneqq \diag \bigl( u \bigl)
\\
    Q(u) & \coloneqq  \mathbbm{1} \cdot  \abs{u}_{1,\epsilon}^{-2}  \cdot \bigl( \abs{u}_{1,\epsilon} \phi   - \inner{u,\phi} \abs{u}_{1,\epsilon}^{-1}  u \bigr)^\intercal  \,.
\end{align*}
Note  that $T(u)$, $U(u)$ are  diagonal and therefore self-adjoint.

\end{itemize}

Using the  ingredients computed above, we obtain the following proposed Algorithm~\ref{alg1}  for generating a sequence
 $(\uu^n, \bar \uu^n, \vv^n, \ww^n)_{n\in \N} \in (\Hr^K \times \Hr^K \times \Hr^{2K} \times \Hr^{2K} )^\N$, where  $\uu^n$ approximates the segmentation mask and $\bar \uu^n, \vv^n, \ww^n$ are auxiliary quantities derived from  the PDHG algorithm. The application of the convergence analysis \citep{valkonen2014primal} requires the Aubin property as well as a smallness condition for the dual variable.  It is an open problem whether these properties are satisfied for our functional.

\begin{algorithm}
\caption{Proposed segmentation algorithm based on  minimizing the reduced energy functional \eqref{eq:minD}.}\label{alg1}
\SetAlgoLined
 \textsc{Input}: Feature map $\ff \in \Hr^K$

\SetAlgoLined
 \textsc{Parameters}: $\lambda,\sigma, \tau, \theta$

 \SetAlgoLined
\textsc{Initialization}:  $\vv^0, \ww^0 \in  \Hr^{2K}$, $\uu^0, \bar \uu^0 \in \Hr^K$

 \For{$n = 0,\dots, $}{

 \SetAlgoLined $\vv^{n+1} \gets P_{2, \infty} ( \vv^{n+1} +  \sigma \gradd \bar \uu^n  )$

  \SetAlgoLined $\ww^{n+1} \gets P_{\infty, \infty} ( \ww^n   + \sigma  \diffD(\bar \uu^n) )$

 \SetAlgoLined $\uu^{n+1} \gets P_\U( \uu^{n}  -  \tau  \gradd^\intercal \vv^{n+1}  -  \tau  (D\diffD)(\uu^n)\ww^{n+1}   ) $

 \SetAlgoLined$\bar \uu^{n+1} \gets \uu^{n+1} + \theta(\uu^{n+1}-\uu^{n})$.
 }

\end{algorithm}

\section{Experimental Results}\label{sec5}

In this section, we present several experiments with the proposed segmentation framework and provide comparison with other variational segmentation methods.  The selected feature maps include  color filters, Gabor filters, and simple windowing techniques. The code for all presented numerical examples  can be found at \url{https://github.com/Nadja1611/Lifting-based-variational-multiclass-segmentation}. For specific parameter settings we refer the interested reader to the code provided there. We use the image processing toolbox scikit-image to create the Gabor filter banks. The parameters for all feature maps used are shown in Table~ \ref{tab:gab1} and Table~\ref{tab:gab2}.

\begin{table}[htb!]
 \caption{ Parameters for feature maps in Fig.~\ref{fig:motivation} (butterfly) and Fig.~\ref{fig:leopard} (leopard). Gabor  filters in the first column are summed to obtain a feature image. The expression $R\geq 230$ denotes a thresholding operation of the red channel in the RGB image.} \label{tab:gab1}
\centering
    \begin{tabular}{|l|l|l|}
       \toprule
         Region  & $\phi_1$ &  $\phi_2$   \\
        \midrule
Butterfly & $n=3,4,5,6 \colon (0, 2^{n+1/2}/256)$     &  $R \geq 230$\\
&  $n=4,6 \colon (\pi/4, 2^{n+1/2}/256)$  &
\\
         \midrule
  Leopard &    $n=4,5,6 \colon (0, 2^{n+1/2}/256)$    &  \\
        &  $n=4,6 \colon (\pi/4, 2^{n+1/2}/256 )$   & \\
        &  $n=5,7 \colon (\pi/2, 2^{n+1/2}/256 ) $   &  \\
         \bottomrule
    \end{tabular}
\end{table}

\subsection{Texture-based multi-class segmentation}

In this subsection, we consider segmentation examples based on  texture information. For that purpose, we tune multiple Gabor filters~\citep{jain1991unsupervised}  with different spatial frequencies and orientations to capture texture in separate channels.  Gabor filters are a  special class of bandpass filters that can be viewed as a sinusoidal signal with a specific frequency and orientation modulated by a Gaussian wave.  Here we  follow  \citep{hammouda2000texture}, in order to extract feature maps  $\phi_1, \phi_2, \dots,\phi_K$ based on Gabor filter banks covering the spatial-frequency domain.

 \begin{figure}[htb!]
\centering
    \includegraphics[trim={1.2cm 1cm 0.3cm 0.5cm},clip,width=0.24\columnwidth, height=0.24\columnwidth]{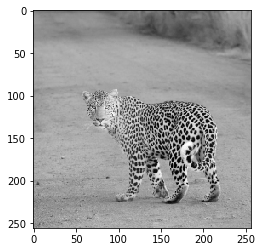}
     \includegraphics[trim={1.2cm 1cm 0.3cm 0.5cm},clip,width=0.24\columnwidth, height=0.24\columnwidth]{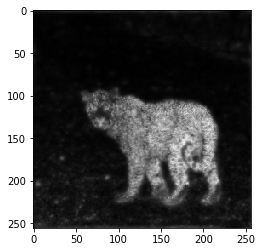}
\includegraphics[trim={1.2cm 1cm 0.3cm 0.5cm},clip,width=0.24\columnwidth, height=0.24\columnwidth]{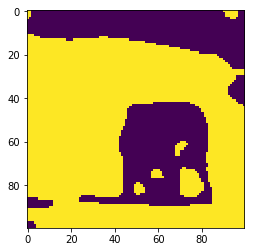}
    \includegraphics[trim={1.2cm 1cm 0.3cm 0.5cm},clip,width=0.24\columnwidth, height=0.24\columnwidth]{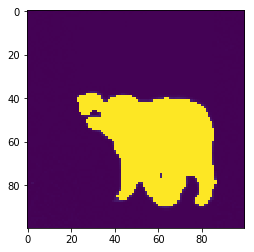}
        \caption{\textbf{Importance of pre-filtering.} From left to right: original  image, feature map representing a linear combination of Gabor filtered images, 
Algorithm~\ref{alg1}} applied to the original  input, and  Algorithm~\ref{alg1} applied to the feature map.\label{fig:leopard}
        \end{figure}

\paragraph{Importance of pre-filtering}

We start by a simple example that demonstrates the relevance of pre-filtering. The first two pictures  in Fig.~\ref{fig:leopard}  show the original image  and the  feature map obtained  by applying a Gabor filter (see Tab-\ref{tab:gab1}), where the aim is to segment  the Leopard. The last two pictures in Fig.~\ref{fig:leopard}  show the resulting minimizers  of the one-channel version of \eqref{eq:func2} applied to the original image and  the feature map, respectively. Clearly,  the segmentation map obtained  from the filtered image better captures the Leopard to be segmented. It is also worth mentioning that the filtering approach comes with great flexibility. By applying different filters, one can emphasis  different structures and different scales in order to  target specific image content.

\begin{table}[htb!]
    \caption{Pre-filtering parameters for the texture images; again the feature maps are obtained by summing all filtered images. \label{tab:gab2}}
    \centering
\begin{tabular}{| l |l |l| l |r |}
       \toprule
         Region  & $\phi_1$ &  $\phi_2$ & $\phi_3$ & $\phi_4$  \\
        \midrule
Brodatz 3 &
$n=2,3 \colon (0,\frac{2^{n+1/2}}{256})$  &
$n=4,5 \colon (\frac{\pi}{4}, \frac{2^{n+1/2}}{256}) $     &
 $n=1,2 \colon \frac{\pi}{4}, \frac{2^{n+1/2}}{256}) $    & (none)
 \\
&  $ n=4 \colon (\frac{\pi}{4}, \frac{2^{n+1/2}}{256})$
&  $ n=3,5 \colon (\frac{3\pi}{4}, \frac{2^{n+1/2}}{256})$ &
    $ n=2 \colon (\frac{\pi}{2}, \frac{2^{n+1/2}}{256}) $ &
\\
        &  $n=4 \colon (\frac{3\pi}{4}, \frac{2^{n+1/2}}{256})$
        &   $ n=1 \colon (\frac{3\pi}{4}, \frac{2^{n+1/2}}{256})$ &  &\\
         \midrule
 Brodatz 5 &
 $n=4 \colon (0,\frac{2^{n+1/2}}{256})$   &
 $n=4 \colon (0,\frac{2^{n+1/2}}{256})$      &
 $n=1,2 \colon (0, \frac{2^{n+1/2}}{256})$    &
 $1- \phi_{1}$
 \\ &
 $n = 4 \colon (\frac{\pi}{4}, \frac{2^{n+1/2}}{256})$   &
 $n = 3,4 \colon (\frac{\pi}{4}, \frac{2^{n+1/2}}{256})$   &
 $n = 2 \colon (\frac{\pi}{4}, \frac{2^{n+1/2}}{256})$ &
 $- \phi_{2}$
 \\
 &
 $n=4 \colon (\frac{\pi}{2}, \frac{2^{n+1/2}}{256})$   &
 $n=3 \colon (\frac{3\pi}{4}, \frac{2^{n+1/2}}{256})$    &
 &
 $- \phi_{3}$ \\
        &  $n=4 \colon (\frac{3\pi}{4}, \frac{2^{n+1/2}}{256})$
        &    & & \\
         \bottomrule
    \end{tabular}
\end{table}

\begin{figure}[htb!]
\begin{center}
\includegraphics[trim={1.2cm 1cm 0.3cm 0.5cm},clip,width=0.24\columnwidth, height=0.24\columnwidth]{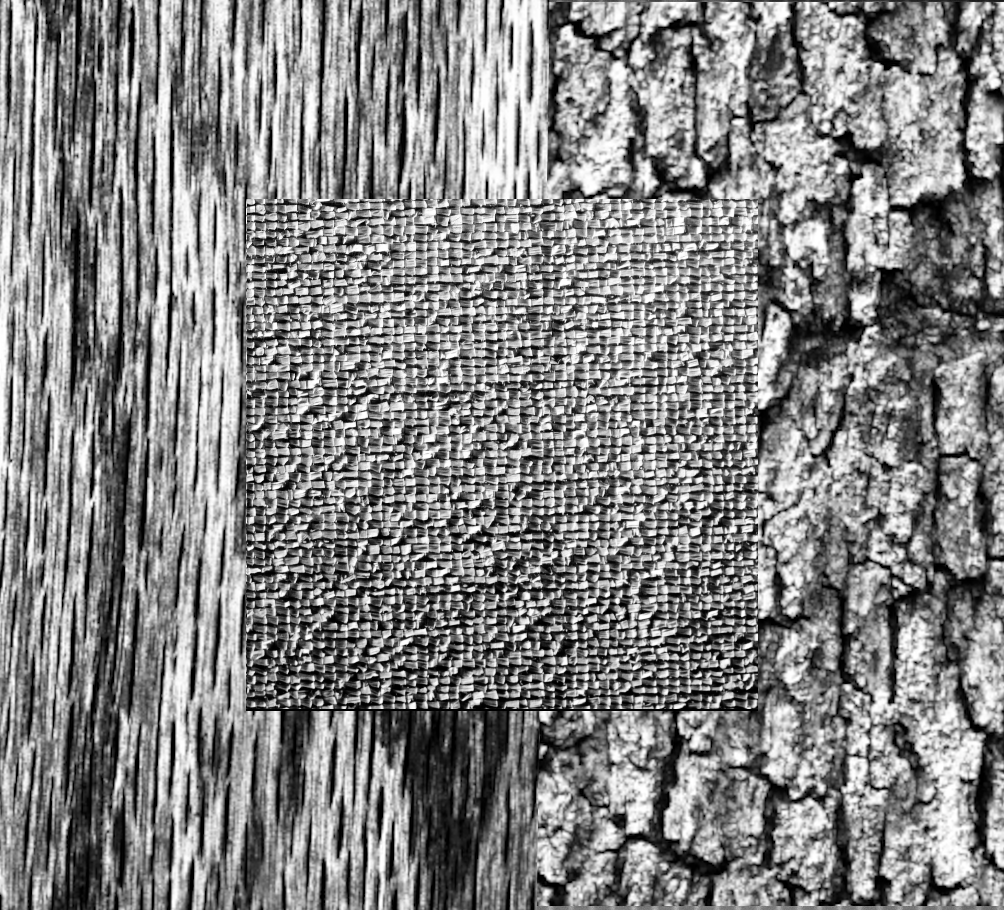}
\includegraphics[trim={1.2cm 1cm 0.3cm 0.5cm},clip,width = 0.24\columnwidth, height=0.24\columnwidth]{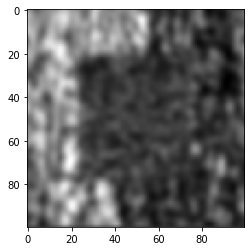}
\includegraphics[trim={1.2cm 1cm 0.3cm 0.5cm},clip,width = 0.24\columnwidth, height=0.24\columnwidth]{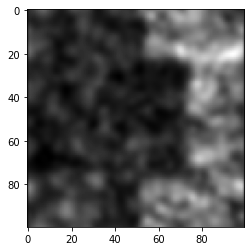}
\includegraphics[trim={1.2cm 1cm 0.3cm 0.5cm},clip,width = 0.24\columnwidth, height=0.24\columnwidth]{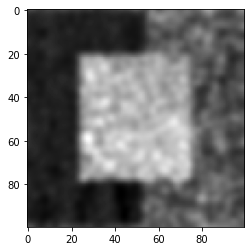}
\\[0.2em]
\includegraphics[trim={1.2cm 1cm 0.3cm 0.5cm},clip,width = 0.24\columnwidth, height=0.24\columnwidth]{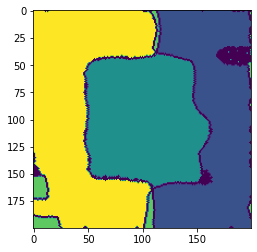}
\includegraphics[trim={1.2cm 1cm 0.3cm 0.5cm},clip,width = 0.24\columnwidth, height=0.24\columnwidth]{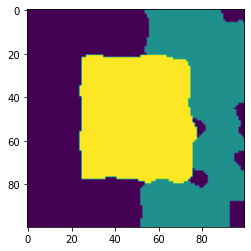}
\includegraphics[trim={1.2cm 1cm 0.3cm 0.5cm},clip,width = 0.24\columnwidth, height=0.24\columnwidth]{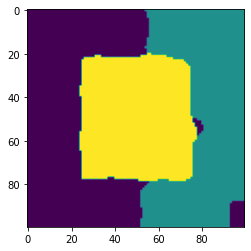}
\includegraphics[trim={1.2cm 1cm 0.3cm 0.5cm},clip,width = 0.24\columnwidth, height=0.24\columnwidth]{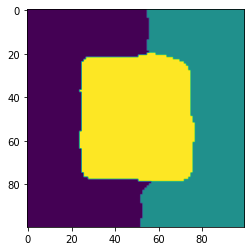}\\[0.2em]
   \includegraphics[width=0.7\columnwidth]{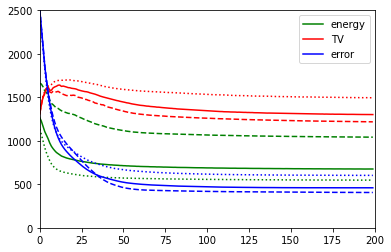}
   \caption{\textbf{Three-texture example.} First row: Input image  and extracted  feature channels. Second row: Result with MMCV and}  minimizers  for $\lambda = 0.1, 0.2, 0.5$ computed with Algorithm~\ref{alg1}. Bottom: Evolution of TV, energy and absolute error over 200 iterations  for $\lambda = 0.1$ (dotted), $\lambda=0.2$ (solid) and $\lambda = 0.5$ (dashed).
\label{brodatz}
\end{center}
\end{figure}

\paragraph{Three-texture example} The next example  considers the segmentation of a texture image (top left picture in Fig.~\ref{brodatz}) consisting of three different Brodatz textures.   The remaining pictures in the top row show the feature maps extracted via Gabor filtering. The second  row  shows the targeted  ground truth segmentation and the minimizers computed with Algorithm~\ref{alg1} for different regularization parameters $\lambda$. The first image in the second row shows the results for MMCV with the regularization parameter $\lambda=0.4$, which has been empirically shown to give the best results. We compared different strategies for MMCV. For example, we used the grayscale image directly, which did not work. For this reason, we used the extracted feature maps in Fig.~\ref{brodatz} and treated them as three different input channels. The second through fourth images in the second row show the minimizers computed using Algorithm~\ref{alg1} for different regularization parameters $\lambda$.
The bottom row shows the evolution of energy, the TV semi-norm and the absolute error compared to the ground truth depending on the number of iterations.  {These results suggest} that the parameter $\la$ not only accounts for the noise but also acts as a way to select a specific resolution of the segmentation task.

\begin{table}[htb!]
    \caption{Quantitative\label{tab:quant} evaluation of segmentation results. These metrics are each calculated as means of the results obtained for different regions.}
\centering
\begin{tabular}{|l|l|l|l|l|l|}
	\toprule
	Method & Dice & Accuracy  &Specificity &   Recall & Precision
	\\
	\midrule
	\multicolumn{6}{|c|}{Brodatz 3} \\
	\midrule
        \textsc{Ours} & 0.956 & 0.974  & 0.980 & 0.964 & 0.954\\
        \textsc{MMCV} & 0.924 & 0.954  & 0.980 & 0.908& 0.950\\
         \midrule
	\multicolumn{6}{|c|}{Abscess} \\
	\midrule
        \textsc{Ours} & 0.739 & 0.975  & 0.989 & 0.649& 0.930\\
       \textsc{MMCV} & 0.566 & 0.941  & 0.980 & 0.594& 0.573 \\
         \bottomrule
\end{tabular}
    
\end{table}

\subsection{Comparison with other algorithm}

Next, we present comparison with other variational segmentation methods. Specifically, we  compare our method with channel-wise Chan-Vese segmentation~\citep{chan2000active}, its convex relaxation~\citep{chan2006algorithms}, and with the {multichannel multiclass Chan-Vese (MMCV) model} \citep{vese2002multiphase}. A quantitative evaluation of the compared algorithms can be found in {Table~\ref{tab:quant}. In these examples, our method performs better than MMCV.

\begin{figure}[htb!]
\centering
\includegraphics[trim={1.2cm 1cm 0.3cm 0.5cm},clip,width = 0.19\columnwidth, height=0.19\columnwidth]{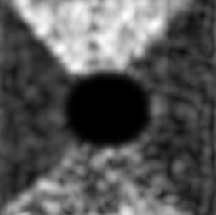}
\includegraphics[trim={1.2cm 1cm 0.3cm 0.5cm},clip ,width = 0.19\columnwidth, height=0.19\columnwidth]{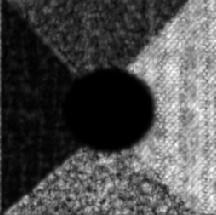}
\includegraphics[trim={1.2cm 1cm 0.3cm 0.5cm},clip,width = 0.19\columnwidth, height=0.19\columnwidth]{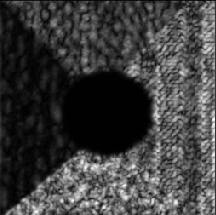}
\includegraphics[trim={1.2cm 1cm 0.3cm 0.5cm},clip,width = 0.19\columnwidth, height=0.19\columnwidth]{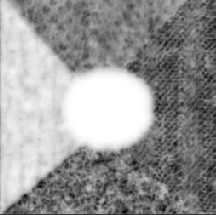}
\includegraphics[trim={1.2cm 1cm 0.3cm 0.5cm},clip,width = 0.19\columnwidth, height=0.19\columnwidth]{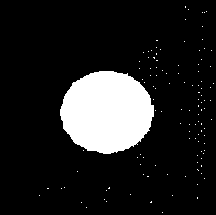}
  \\[0.2em]
\includegraphics[trim={1.2cm 1cm 0.3cm 0.5cm},clip ,width = 0.19\columnwidth, height=0.19\columnwidth]{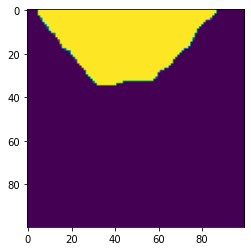}
\includegraphics[trim={1.2cm 1cm 0.3cm 0.5cm},clip ,width = 0.19\columnwidth, height=0.19\columnwidth]{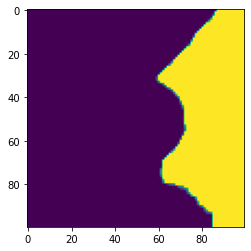}
\includegraphics[trim={1.2cm 1cm 0.3cm 0.5cm},clip,width = 0.19\columnwidth, height=0.19\columnwidth]{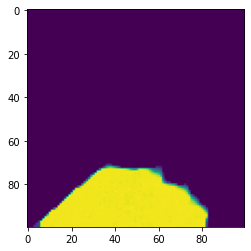}
\includegraphics[trim={1.2cm 1cm 0.3cm 0.5cm},clip,width = 0.19\columnwidth, height=0.19\columnwidth]{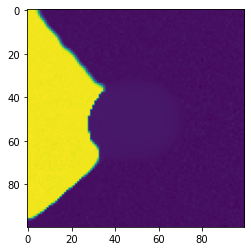}
\includegraphics[trim={1.2cm 1cm 0.3cm 0.5cm},clip,width = 0.19\columnwidth, height=0.19\columnwidth]{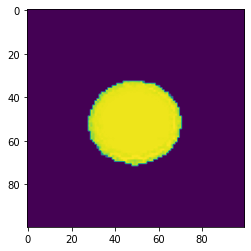}
  \\[0.2em]
\includegraphics[trim={1.2cm 1cm 0.3cm 0.5cm},clip,width = 0.19\columnwidth, height=0.19\columnwidth]{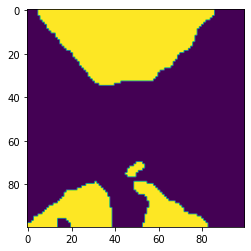}
\includegraphics[trim={1.2cm 1cm 0.3cm 0.5cm},clip,width = 0.19\columnwidth, height=0.19\columnwidth]{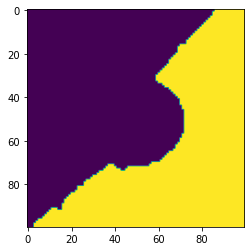}
\includegraphics[trim={1.2cm 1cm 0.3cm 0.5cm},clip,width = 0.19\columnwidth, height=0.19\columnwidth]{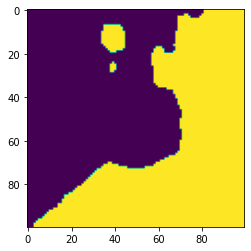}
\includegraphics[trim={1.2cm 1cm 0.3cm 0.5cm},clip,width = 0.19\columnwidth, height=0.19\columnwidth]{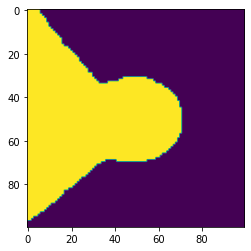}
\includegraphics[trim={1.2cm 1cm 0.3cm 0.5cm},clip,width = 0.19\columnwidth, height=0.19\columnwidth]{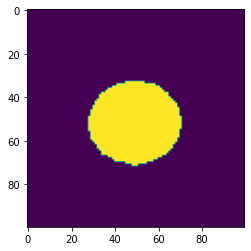}
  \\[0.2em]
\includegraphics[trim={1.2cm 1cm 0.3cm 0.5cm},clip,width = 0.19\columnwidth, height=0.19\columnwidth]{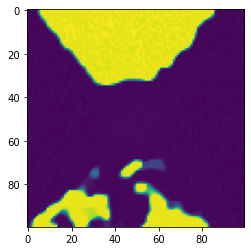}
\includegraphics[trim={1.2cm 1cm 0.3cm 0.5cm},clip ,width = 0.19\columnwidth, height=0.19\columnwidth]{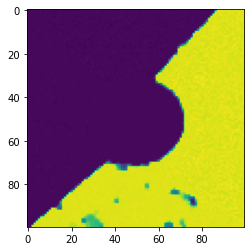}
\includegraphics[trim={1.2cm 1cm 0.3cm 0.5cm},clip,width = 0.19\columnwidth, height=0.19\columnwidth]{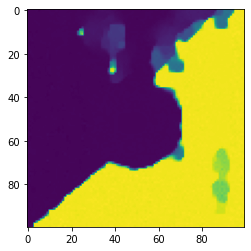}
\includegraphics[trim={1.2cm 1cm 0.3cm 0.5cm},clip,width = 0.19\columnwidth, height=0.19\columnwidth]{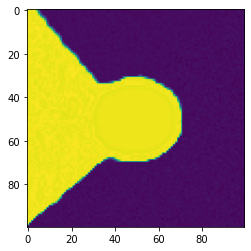}
\includegraphics[trim={1.2cm 1cm 0.3cm 0.5cm},clip,width = 0.19\columnwidth, height=0.19\columnwidth]{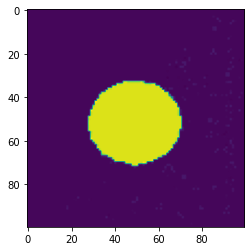}
\caption{\textbf{Comparison with channel-wise (convex) Chan-Vese.} First row: Feature maps obtained via Gabor filtering and thresholding.
Second row:
Proposed method. Third row:  Channel wise  Chan-Vese. Fourth row: Channel wise convex Chan-Vese. All  experiments use  $\lambda = 0.1$.}\label{multivsconvex}
  \end{figure}

A first comparison is shown in Fig.~\ref{multivsconvex}, where we use $K=5$ feature maps via Gabor filtering  and  simple thresholding.  The pre-filtered images slightly highlight the five different texture regions, but again there are many overlaps and  holes. The minimizer of the proposed functional~\eqref{eq:func} is depicted in the second row of  Fig.~\ref{multivsconvex}. The third and fourth row, respectively,   show results with the  Chan-Vese model and its  convex relaxation applied to each channel separately.
Results clearly show the importance of the constraint preventing the results from overlapping.

\begin{figure}[htb!]
\centering
\includegraphics[trim={1.2cm 1cm 0.3cm 0.5cm},clip,width= 0.32\columnwidth,height= 0.25\columnwidth]{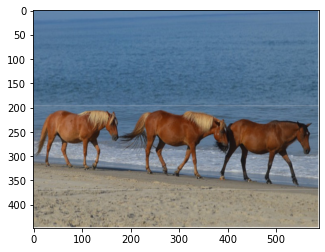}
\includegraphics[trim={1.2cm 1cm 0.3cm 0.5cm},clip,width= 0.32\columnwidth,height= 0.25\columnwidth]{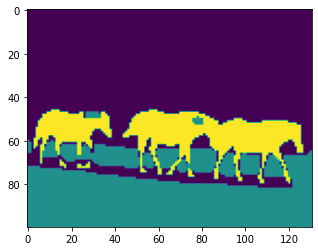}
 \includegraphics[trim={1.2cm 1cm 0.3cm 0.5cm},clip,width= 0.32\columnwidth,height= 0.25\columnwidth]{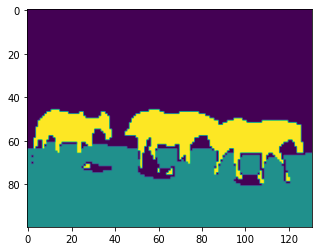}  \\[0.2em]
\includegraphics[trim={1.2cm 1cm 0.3cm 0.5cm},clip,width= 0.32\columnwidth,height= 0.25\columnwidth]{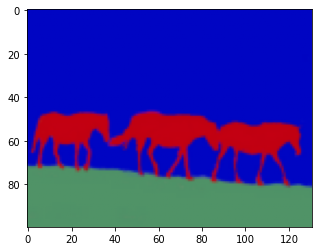}
\includegraphics[trim={1.2cm 1cm 0.3cm 0.5cm},clip,width= 0.32\columnwidth,height= 0.25\columnwidth]{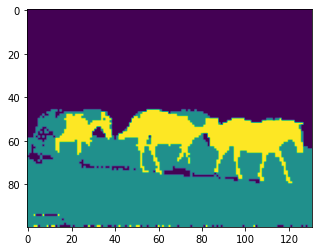}
\includegraphics[trim={1.2cm 1cm 0.3cm 0.5cm},clip,width= 0.32\columnwidth,height= 0.25\columnwidth]{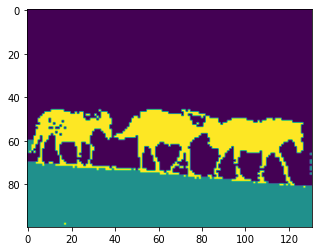}
\caption{\textbf{Comparison with MMCV:} Left: input image and groundtruth segmentation}. The remaining pictures  show segmentation results for two  different values of the regularization parameter using the proposed method (middle) and the MMCV model (right).\label{multiphase}
\end{figure}

Another comparison is presented  Fig.~\ref{multiphase}.  In this example, the original input is an RGB image, so we have three channels serving as input of the proposed functional. The middle column shows results with  the proposed method  and the right column results with the MMCV model using the  MATLAB-implementation provided by  Wu~\citep{alg}. Both methods achieve acceptable results, although those of the proposed method look slightly smoother.

\subsection{Medical applications}

The proposed framework can be  directly applied to various medical imaging applications where the different channels naturally result from different imaging techniques. Our multichannel multiclass functional uses the information contained in the different categories of MRI and CT images or images resulting from different modalities, and can therefore divide the input image naturally into $K$ non-overlapping sub-regions.

As a first example, we  consider a neuroradiological application,  where we aim for dividing MRI images showing an abscess into the regions healthy,  abscess and  edema. The left column in Fig.~\ref{fig:abszess} shows three channels of an MRI dataset using diffusion weighted imaging (DWI), apparent diffusion coefficient (ADC), and T2 sequence. As preprocessing, the ADC maps and T2 images were windowed and standardized to have intensity values between $[0,1]$. The second column shows the segmentations obtained by MMCV. Here we chose a weighted linear combination of the three images as the input image. This strategy was found to provide better results than treating the three images as separate channels. We selected the regularization parameter empirically so that it led to the best results. The remaining columns of Fig.~\ref{fig:abszess} show the segmentation masks  corresponding to  four different regularization parameters. Smaller values of $\lambda$ result in quite inaccurate results containing many false positives. The results in the last row, which were obtained by setting the regularization parameter to $\lambda = 0.2$, show an improvement compared to the ones in the previous columns.
Quantitative evaluation metrics are summarized in Table~\ref{tab:quant}.

\begin{figure}[htb!]
\centering
\includegraphics[trim={1.2cm 1cm 0.3cm 0.5cm},clip,width= 0.19\columnwidth,height= 0.19\columnwidth]{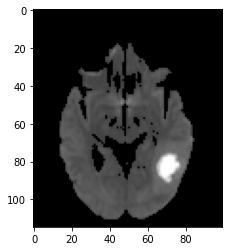}
\includegraphics[trim={1.2cm 1cm 0.3cm 0.5cm},clip,width= 0.19\columnwidth,height= 0.19\columnwidth]{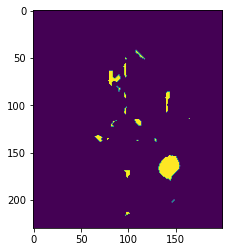}
\includegraphics[trim={1.2cm 1cm 0.3cm 0.5cm},clip,width= 0.19\columnwidth,height= 0.19\columnwidth]{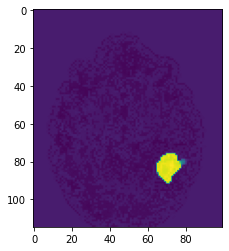}
\includegraphics[trim={1.2cm 1cm 0.3cm 0.5cm},clip,width= 0.19\columnwidth,height= 0.19\columnwidth]{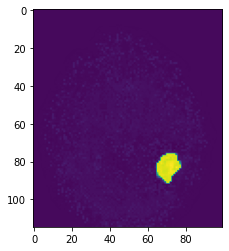}
\includegraphics[trim={1.2cm 1cm 0.3cm 0.5cm},clip,width= 0.19\columnwidth,height= 0.19\columnwidth]{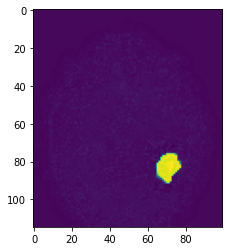}\\
\includegraphics[trim={1.2cm 1cm 0.3cm 0.5cm},clip,width= 0.19\columnwidth,height= 0.19\columnwidth]{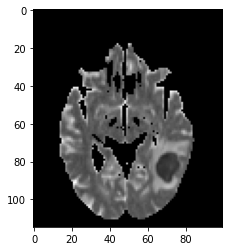}
\includegraphics[trim={1.2cm 1cm 0.3cm 0.5cm},clip,width= 0.19\columnwidth,height= 0.19\columnwidth]{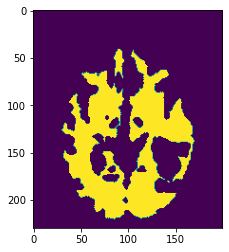}
\includegraphics[trim={1.2cm 1cm 0.3cm 0.5cm},clip,width= 0.19\columnwidth,height= 0.19\columnwidth]{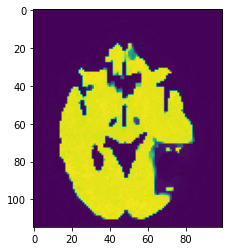}
\includegraphics[trim={1.2cm 1cm 0.3cm 0.5cm},clip,width= 0.19\columnwidth,height= 0.19\columnwidth]{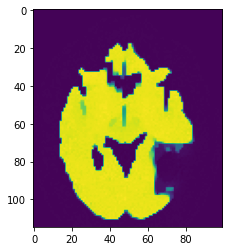}
\includegraphics[trim={1.2cm 1cm 0.3cm 0.5cm},clip,width= 0.19\columnwidth,height= 0.19\columnwidth]{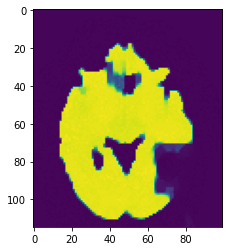} \\
\includegraphics[trim={1.2cm 1cm 0.3cm 0.5cm},clip,width= 0.19\columnwidth,height= 0.19\columnwidth]{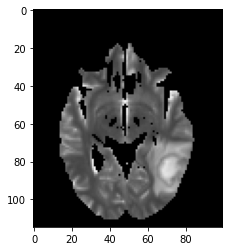}
\includegraphics[trim={1.2cm 1cm 0.3cm 0.5cm},clip,width= 0.19\columnwidth,height= 0.19\columnwidth]{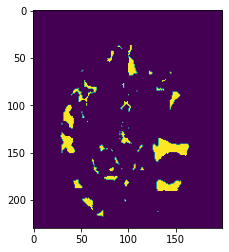}
\includegraphics[trim={1.2cm 1cm 0.3cm 0.5cm},clip,width= 0.19\columnwidth,height= 0.19\columnwidth]{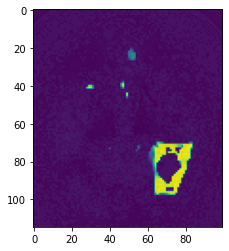}
\includegraphics[trim={1.2cm 1cm 0.3cm 0.5cm},clip,width= 0.19\columnwidth,height= 0.19\columnwidth]{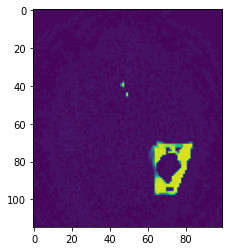}
\includegraphics[clip,width= 0.19\columnwidth,height= 0.19\columnwidth]{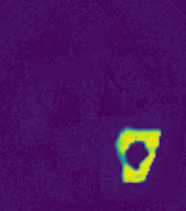}
\caption{\textbf{Abscess in multichannel MRI data.} Left column:  DWI, ADC and T2 image of the same brain. Second column: results with  MMCV. Remaining columns: Proposed method for different choices of  $\lambda$. From top to bottom, the segmentation masks for abscess (top), healthy brain tissue (middle) and the edema (bottom) are visualized..}\label{fig:abszess}
\end{figure}

For the second medical example,  we use image from the brain tumor segmentation (BRATS) challenge 2015 dataset; see https://www.smir.ch/BRATS/Start2015. The channels in this case consist of  contrast-enhanced T1-weighted (T1c), T2 and fluid attenuated inversion recovery (FLAIR) images. Channels contain complementary  information allowing  accurate diagnosis and quantification of tumor growth.  The top row in Fig.~\ref{fig:images} shows the input images highlighting a different region of the tumor. Again, we exploit the information contained in different sequences  by employing them as separate input channels. The proposed method is thus able to use this complementary information to delineate the different tissues and demonstrates solid results for this concrete example from medicine. The fact that the proposed energy functional can be applied directly to the given images, makes it particularly suitable for medical image segmentation.

\begin{figure}[htb!]
\centering
\includegraphics[trim={1.2cm 1cm 0.3cm 0.5cm},clip,width= 0.32\columnwidth,height= 0.32\columnwidth]{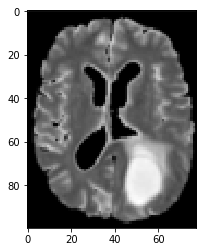}
\includegraphics[trim={1.2cm 1cm 0.3cm 0.5cm},clip,width= 0.32\columnwidth,height= 0.32\columnwidth]{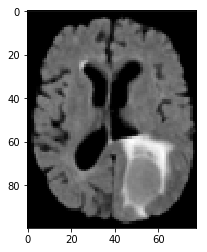}
\includegraphics[trim={1.2cm 1cm 0.3cm 0.5cm},clip,width= 0.32\columnwidth,height= 0.32\columnwidth]{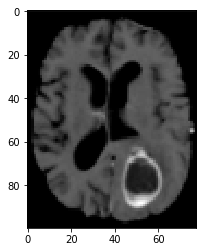} \\[0.1em]
\includegraphics[trim={1.2cm 1cm 0.3cm 0.5cm},clip,width= 0.32\columnwidth,height= 0.32\columnwidth]{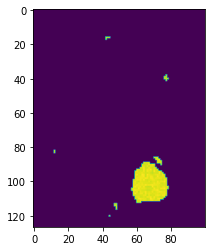}
\includegraphics[trim={1.2cm 1cm 0.3cm 0.5cm},clip,width= 0.32\columnwidth,height= 0.32\columnwidth]{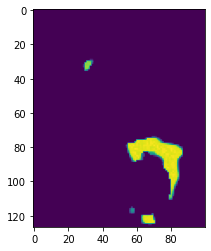}
\includegraphics[trim={1.2cm 1cm 0.3cm 0.5cm},clip,width= 0.32\columnwidth,height= 0.32\columnwidth]{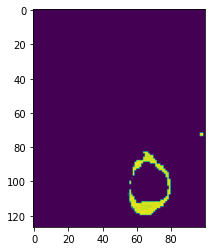}
\caption{\textbf{Tumor growth example.} Top row:  T1c, T2 and FLAIR from same brain from  BRATS dataset. Bottom row: Segmentation results for tumor(left), edema (middle) and necrosis (right).}\label{fig:images}
\end{figure}

\section{Conclusion}\label{sec6}

In this paper, we have proposed a framework for variational image segmentation using feature map lifting and minimizing a multichannel segmentation functional. Input channels can be given either in a natural form, such as RGB images, or can be extracted by some pre-filtering method. We demonstrated the effectiveness of our method on several images, such as texture images or multi-modal medical images, and achieved convincing results also in comparison with related variational  approaches. The method can distinguish a number of different regions, and is particularly suitable for applications in medical imaging. As main theoretical result, we have shown existence, stability and convergence  with respect to the distorted input data. For future extension of the proposed framework, the combination with deep learning is planned, in particular, the feature maps extracted in this paper by means of pre-filtering might  be obtained by a neural network. The combination of the strengths of modern deep learning and classical energy based segmentation methods could further improve the existing results, and enable more complex problems to be solved.  In addition, a detailed comparison with other segmentation methods that also utilize feature lifting will be conducted.

The identification of appropriate feature maps is the main step in the proposed framework and is considered as its main current limitation. In order to obtain practically realistic results, the specific lifting has a significant impact on the final segmentation results. In this work, we focus on textured images, where the appropriate pre-filtering can be selected manually.   In a recent follow-up work, we have proposed an automated lifting using convolutional networks {\citep{gruber2023variational}}. Modern self-supervised deep learning methods, together with appropriate loss functions, can be used to obtain reasonable image decompositions. For example, in \citep{gruber2023variational}, a loss function that minimizes the correlation between different feature maps, along with a norm constraint on the feature maps, has been applied.  Even with this, it is still somewhat unclear which segmentation is actually targeted by the lifting. To overcome this limitation, in future work. we plan to incorporate user guidance by allowing the user to roughly select representative regions for the different classes provided, in order to provide realistic results. It should also be noted that in the current setting, the pre-filtering should ensure a good separation of the channels.

\section*{Acknowledgments}
This study is supported by VASCage – Research Centre on Vascular Ageing and Stroke. As a COMET centre VASCage is funded within the COMET program - Competence Centers for Excellent Technologies by the Austrian Ministry for Climate Action, Environment, Energy, Mobility, Innovation and Technology, the Austrian Ministry for Digital and Economic Affairs and the federal states Tyrol, Salzburg and Vienna.

\section*{List of Acronyms}

\begin{tabular}{p{2cm}p{10cm}}
    ADC & Apparent diffusion coefficient \\
    BRATS & Brain tumor segmentation \\
    BV  & Bounded variation \\
    CT & Computed tomography \\
    DWI & Diffusion weighted imaging \\
    FLAIR & Fluid attenuated inversion recovery \\
    MMCV & Multichannel multiclass Chan-Vese \\
    MRI & Magnetic resonance imaging \\
    PDHG & Primal-dual hybrid gradient \\
    RGB & Red green blue \\
    TV & Total variation \\
    T1c & Contrast-enhanced T1-weighted image 
\end{tabular}

\bibliographystyle{elsarticle-num} 
\bibliography{main.bib}

\end{document}